\documentclass[journal]{IEEEtran}
\usepackage{xcolor}
\usepackage{booktabs}
\usepackage{threeparttable}
\usepackage{multirow}
\usepackage{amsmath,amsfonts,amssymb, mathrsfs, epsfig,epstopdf,url,array} 
\usepackage{graphicx} 
\usepackage{float} 
\usepackage{subfigure} 
\usepackage{siunitx}
\usepackage{makecell}
\usepackage{cite}
\usepackage{ulem}
\usepackage{tikz}
\usepackage{bm}
\usepackage{amsthm}
\usetikzlibrary{arrows.meta,calc,positioning}
\usepackage[ruled,linesnumbered]{algorithm2e}
\SetKw{Continue}{continue}
\SetKw{Break}{break}
\SetKwData{KwData}{Initialize:}
\SetKwComment{Comment}{/* }{ */}
\SetKwComment{SingleComment}{// }{}
\SetKwComment{TriComment}{$\triangleright$\ }{}

\newtheorem{mytheo}{Theorem}
\newtheorem{myrek}{Remark}

\newtheorem{myppt}{Property}

\begin{document}

\title{{Global Optimality of Pose Graph Optimization}}

\author{Shuan Ding*, Ruishuang Chen*, Mengfan  Cao, and Zaiyue Yang,~\IEEEmembership{Senior Member,~IEEE}
\thanks{
	
	* S. Ding and R. Chen contribute equally to this work.
	
	\textit{(Corresponding author: Zaiyue Yang.)}
		
S. Ding, M. Cao, and Z. Yang are with the Guangdong Provincial Key Laboratory of FullyActuated System Control Theory and Technology, and the School of Automation and Intelligent Manufacturing, Southern University of Science and Technology, Shenzhen 518055, China (e-mails: dingsa2022@mail.sustech.edu.cn; 12331364@mail.sustech.edu.cn; yangzy3@sustech.edu.cn).

R. Chen is with the Department of Computer Science and Engineering, Southern University of Science and Technology, Shenzhen 518055, China (e-mail: chenrs2020@mail.sustech.edu.cn).


}
}

\maketitle

\begin{abstract}
Pose Graph Optimization (PGO) is a classical and widely applied problem in state estimation for robotics, playing a fundamental role in localization, mapping, and navigation. 
Yet, it is nonconvex in nature and may have many local minimizers, thus difficult to analyze.
In this paper, we investigate the optimization landscape 
and establish a sufficient condition of global optimality for PGO.
The ultimate goal of this paper is to theoretically confirm that the optimization landscape is benign and `attractive' around the global minimizer, if noise is smaller than an explicit and \textit{a priori} bound.
To reveal such an insight, this paper develops three main technical results.
First, we identify a proper basin of attraction (BOA) containing global minimizer, and characterize its geometric property by showing that its radius is controlled by noise in an explicit way. 
Then, by exploring the first order optimality conditions of PGO, we prove that a local minimizer is globally optimal if a noise-controlled criterion is satisfied. 
Finally, by combining above two results, we prove that one and only one local minimizer exists in the BOA and this minimizer is globally optimal, when the noise is smaller than an explicit bound that can be verified \textit{a priori}. 
Numerical simulations  validate the proposed theory under various random  cases, showing that the obtained conditions provide an effective criterion for certifying the global optimality of PGO.
\end{abstract}

\begin{IEEEkeywords}
Pose graph optimization (PGO), Global optimality, Bounds analysis, Nosie analysis
\end{IEEEkeywords}

\section{Introduction}

\subsection{Motivation}

Pose Graph Optimization (PGO) \cite{RN1047} is a core problem for state estimation in the field of robotics, with broad applications in tasks such as simultaneous localization and mapping (SLAM) \cite{RN1139, RN1169,RN1170}, multi-robot coordination \cite{RN1136, RN1138}, autonomous navigation \cite{RN1140}, etc.
Its basic objective is to recover the global poses of a group of nodes from noisy relative pose measurements \cite{RN1141, RN1171}. 
Since the estimated poses directly affect subsequent missions, 
the accuracy and stability of solving  PGO largely determine the overall performance of the entire system \cite{RN1143}. 
Therefore, PGO is not only an engineering problem with broad practical relevance, but also an important theoretical problem that has long attracted significant attention in geometric estimation and optimization \cite{RN1172, RN1173, RN1174}. 

However, PGO is inherently high-dimensional and nonconvex. 
In particular, the rotational variables are typically constrained by group structures or manifold constraints, which endow both the feasible set and the objective function with complicated geometric characteristics. 
Intuitively, such complexity may suggest that PGO does not possess a favorable optimization landscape, and as a result, much of the existing research has focused on developing more efficient  algorithms \cite{RN1151, RN1147, RN1145}. 
Meanwhile, practical computation process shows that, although PGO is nonconvex, most algorithms often can converge to the global optimum, and \textit{a posteriori} optimality certificates further support this observation \cite{RN1084, RN1153}. 

Admittedly, such strong performance is made possible in large part by the effectiveness of these algorithms and 
strategies.
However, this phenomenon also inspires a more theoretical line of inquiry: perhaps PGO is benign under certain conditions, so that its nonconvex formulation can in fact be solved correctly, and the obtained solution indeed corresponds to the global optimum. 
Motivated by this perspective, this paper investigates the intrinsic landscape properties of PGO from the viewpoint of its structural characteristics, and establishes analytical conditions for achieving global optimality.

\subsection{Literature Review}

The formulation and solution of  PGO  can be traced back to the seminal work of Lu and Milios in 1997 \cite{RN1154}, and since then it has attracted widespread attention in the robotics community.
Due to the strong engineering background of PGO, existing research has mainly focused on the design of  algorithms. 
The most classical approach is to treat the original problem from the perspective of nonlinear programming (NLP) and solve it via numerical iteration and searching based on a given initial guess \cite{RN1155}.  
These methods mainly include gradient-based methods \cite{RN1156, RN1157}, the Gauss-Newton method \cite{RN1158, RN1160, RN1069}, trust-region methods \cite{RN1161}, and related variants.
Due to the sensitivity of nonlinear iterative methods to initial guess and their potentially low computational efficiency, many researchers have proposed various acceleration strategies, including exploiting matrix sparsity for dimensionality reduction \cite{RN1068, RN1069}, limiting the number of poses to reduce complexity \cite{RN1070, RN1071}, 
utilizing approximate solutions \cite{RN1072}, and introducing a dimensionality reduced objective function \cite{RN1085}.
In addition, since the rotation subproblem in PGO can often be  solved effectively, many works first estimate the rotation and then use the resulting solution to bootstrap the subsequent iterative solver \cite{ RN1072,RN1077,RN1079}, 
thereby improving the optimality of  entire algorithm.

Besides analyzing and solving the problem based on  NLP, 
Carlone \textit{et al.} \cite{RN1047} introduce a new paradigm for PGO from the viewpoint of duality and relaxation.
They first construct a more compact mathematical model of PGO over the complex domain, and then, by exploiting Lagrangian duality, derive both a dual problem and a relaxation in the form of semidefinite programming (SDP) \cite{RN1163}. 
They prove that when the pose graph matrix has the single zero eigenvalue property (SZEP), the duality gap is zero, and the solution to  primal problem can be recovered by appropriately scaling the eigenvector associated with the zero eigenvalue. 
They further develop an SZEP-based solution algorithm, together with an \textit{a posteriori} criterion for certifying whether the obtained solution is globally optimal.

The  work in \cite{RN1047}  is the first to introduce SDP relaxation into the solution of  PGO, 
and it  also inspires subsequent studies to further improve and extend this framework and line of thought.
Rosen \textit{et al.} \cite{RN1082}  investigate PGO from the perspective of synchronization over the special Euclidean group and propose the classical SE-Sync algorithm based on SDP relaxation. By exploiting the low-rank, geometric, and graph-theoretic structures of the original problem, SE-Sync  reduces the SDP to an optimization problem on a low-dimensional Riemannian manifold \cite{RN1164}, and then solves it using a Riemannian truncated-Newton trust-region method. 
In addition, the algorithm is able to certify \textit{a posteriori} whether the estimated solution is globally optimal.
Briales et al. \cite{RN1084} also propose the efficient Cartan-Sync algorithm based on SDP relaxation and Riemannian optimization, which is likewise capable of certifying the global optimality of the estimated solution \textit{a posteriori}. However, Cartan-Sync jointly optimizes rotation and translation, whereas SE-Sync first eliminates the translation variables analytically and then performs the relaxation.
Choudhary \textit{et al.} \cite{RN1087} develop a distributed algorithm for PGO  based on successive over-relaxation and Jacobi over-relaxation, which can assign the computational tasks among different nodes to improve computation efficiency and system robustness. 
Furthermore, Tian \textit{et al.} \cite{RN1088} propose a Riemannian block coordinate descent (RBCD) method based on  sparse SDP relaxation to achieve distributed computation. 
Meanwhile, the authors propose distributed verification and saddle-point escape methods, which can certify the global optimality of estimated results and continue iterative descent from suboptimal critical points \cite{RN1142}.

Existing research has primarily focused on algorithmic design for PGO, leading to numerous efficient solvers that achieve remarkable performance in terms of computational efficiency,   noise robustness, and optimality \cite{RN1160, RN1069, RN1165}. 
However, relatively little attention has been devoted to the theoretical analysis of the structural properties of solutions and their global optimality \cite{RN1056}. 
Most existing work concentrates on \textit{a posteriori} verification criteria—namely, first obtaining a solution via an algorithm and then checking whether the resulting estimate satisfies global optimality. 
It should be emphasized that, given the high-dimensional non-convex nature of PGO, having an explicit 
\textit{a posteriori} optimality certificate is already a profound and significant achievement.  
Nevertheless, empirical observations from practical computation processes reveal that, under most non-adversarial noise scenarios, the estimates produced by current algorithms are often verified as globally optimal. 
This not only reflects the inherent strength of these algorithms but also suggests that PGO may be well-behaved  under certain conditions.

Actually, PGO is  structurally similar  to the widely studied angular synchronization or phase synchronization problems \cite{RN1167, RN1054, RN1166, RN1168}. Both belong to maximum-likelihood estimation problems and both involve unit-norm constraints.
The latter, with the aid of SDP relaxation, has been shown to possess a favorable optimization landscape under certain conditions and admits  \textit{a prior} conditions that guarantee global optimality.
Motivated by these observations and empirical phenomena, this paper investigates the optimization landscape of PGO from the perspective of  structural properties, and establishes analytical conditions for its estimated solutions to achieve global optimality, along with characterizing relevant properties.

\subsection{Contributions of the Paper}

In contrast to most existing studies on PGO, which mainly focus on the design of efficient numerical solvers and \textit{a posteriori} certification of solution, this paper investigates PGO from a theoretical perspective, with particular emphasis on the relationship among graph spectral structure, measurement noise, and global optimality. The main contributions of this paper are summarized as follows:

1) We identify a proper basin of attraction (BOA) that containing the global minimizer, and then derive explicit upper bounds on its scope under both the
$2$-norm and the $\infty$-norm. These bounds quantitatively characterize how the orbit radius depends on the spectral gap and the noise.

2) We characterize the explicit relationship between the optimality and  orbit radius over the entire feasible set.
Building on the SDP relaxation and the Karush-Kuhn-Tucke (KKT) conditions,
we prove that when the orbit radius is sufficiently small and satisfies specific conditions involving the noise matrix, the corresponding feasible solution is globally optimal.


3) By combining the orbit radius bounds with the bounds on the action of noise, we obtain an explicit analytical upper bound on the norm of noise matrix in BOA. 
These conditions are independent of the specific value of orbit radius,
thereby providing a more direct and \textit{a priori} theoretical sufficient condition for analyzing the optimization landscape of PGO.

\subsection{Organization of the Paper}
The organization of the paper is as follows. Section II
formulates the PGO problem. Section III illustrates the scope of BOA. Section IV discusses the optimality conditions of porblem for any solution.
Section V analyzes the optimality conditions within the BOA.
Section VI presents the simulation results  and the conclusions are given
in Section VII.

\subsection{Notation Description}
For ease of exposition, some basic notations are introduced here.
$\mathbb{C}$ denotes the complex domain;
$\mathbb{R}$ denotes the real domain;
$\|*\|$ applies only to matrices  and denotes the spectral norm;
$\|\cdot\|_{2(\infty)}$ denote the $2$-norm ($\infty$-norm) of vector;
$j$ is the complex operator;
$*^H$ denotes the conjugate transpose;
$\Re$ denotes the real part.
$\operatorname{diag} $ constructs a diagonal matrix from a vector, or extracts the diagonal elements from matrix to constructs a diagonal matrix;
$\operatorname{diagV} $ extracts the diagonal elements from a matrix to form a vector.

\begin{figure}[t] 
	\centering 
	\includegraphics[width=0.45\textwidth]{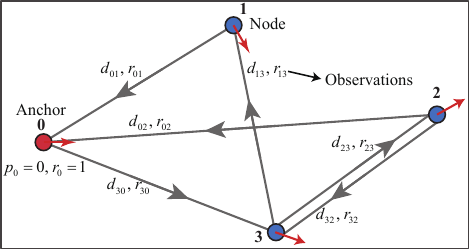} 
	\caption{Illustration of PGO.} 
	\label{fig:model_example}
\end{figure}

\section{Problem Formulation}

As the foundation for theoretical analysis, this section presents the problem formulation of PGO, including its general form, the equivalent form with coordinate variables eliminated, and the gauge-normalized form.

\subsection{General and equivalent formulation of PGO}

In this part, we will  present the general formulation of PGO and its equivalent reformulation. 
We focus on the planar PGO case, which is the most widely studied scenario. 
For brevity without loss of generality, we illustrate the mathematical model of PGO using the instance of multi-robot relative localization, as shown in Fig. \ref{fig:model_example}.

In Fig. \ref{fig:model_example}, $ p_i \in \mathbb{C}$ is the global position of robot $i $ 
($ i = 0,1, \dots, N-1$ and $ N $ denotes the number of robots),
$ r_i \in \mathbb{C}$ is the global orientation, $d_{ij} \in \mathbb{C}$ and $r_{ij} \in \mathbb{C}$ are the relative position and orientation measurements ($i$  measures $ j$), respectively. 
The task of PGO is to recover the position $p_i $ and orientation $r_i$ of each robot in a global coordinate frame based on  relative measurements $d_{ij}, r_{ij}$. 
Since there is no  global reference frame, one robot must be designated as an anchor. Typically, robot $0$ can be  chosen as the anchor, 
thus we have  $p_0 = 0, r_0 = 1$.
It is emphasizing that the anchor merely serves as a unified reference coordinate and does not alter the nature of the problem. Therefore, in the subsequent analysis, we will focus on the generalized anchor-free formulation.

In the noiseless situation, the relative measurements of positions satisfy 
\begin{equation}\label{position_measure_single}
	p_i + r_i  d_{ij} - p_j = 0
\end{equation}
Similarly, the relative measurements of orientations satisfy
\begin{equation} \label{orientation_measure_single}
	r_i r_{ij} - r_j = 0
\end{equation}
To make the formulation more compact, 
we rewrite the  $N_d $ position measurement equations (\ref{position_measure_single})  in matrix form:
\begin{equation}\label{AD-matrix}
	A \mathbf{p}  + D^{\star} \mathbf{r}  = \mathbf{0}
\end{equation}
where $  A \in \mathbb{R}^{N_d \times N}$  denotes the adjacency matrix (composed of entries  $1 ,  -1 ,   0$ , with  $1$  indicating observing and  $-1 $ indicating being observed),
$\mathbf{p} = \left[ p_0, p_1, \dots, p_{N-1} \right]^\top$,
$\mathbf{r} = \left[ r_0, r_1, \dots, r_{N-1} \right]^\top$,
and $D^{\star} \in \mathbb{C}^{N_d \times N}$ contains all the measurements of positions.
The $N_d$ orientation measurement equations (\ref{orientation_measure_single})  
can also be rewritten in  matrix form:
\begin{equation}\label{C-matrix}
	C^{\star}\mathbf{r} = 0
\end{equation}
where $C^{\star} \in \mathbb{C}^{N_d \times N}$ contains all the 
measurements of orientation.
Furthermore, combining  (\ref{AD-matrix}) and (\ref{C-matrix}) can yield the overall block-structured model:
\begin{equation}\label{no-noise-all-model}
	\left[\begin{array}{ll}
		A & D^{\star} \\
		0 & C^{\star}
	\end{array}\right]\left[\begin{array}{l}
		\mathbf{p} \\
		\mathbf{r}
	\end{array}\right]= \mathbf{0}
\end{equation}
Eq. (\ref{no-noise-all-model}) represents the PGO model under ideal condition and solving this linear equation directly yields  $\mathbf{p}$  and  $\mathbf{r}$.

However, the actual measurement process is subject to noise, thus solving (\ref{no-noise-all-model}) becomes a least square  problem (LSP) on the manifold of $N$-torus:
\begin{equation}\label{pgo-lsp}
	\begin{array}{ll}
		\underset{\mathbf{p}, \mathbf{r}}{\operatorname{min}} & \left\|A \mathbf{p}  + \tilde{D} \mathbf{r}\right\|^2_2+\left\|\tilde{C} \mathbf{r} \right\|^2_2 \\
		\text { s.t.} & \left|r_i\right|=1, \quad i=0, 1, \ldots, N-1 
	\end{array}
\end{equation}
where $\tilde{D} $ and $ \tilde{C}$ denote $D^{\star}$ and $ C^{\star} $ with noises, respectively.
Eq. (\ref{pgo-lsp})  is the generalized form of PGO, which is a non-convex quadratically constrained quadratic program (QCQP) with unit-norm constraints.

To make the problem more compact and amenable to analysis, we express the position variables in (\ref{pgo-lsp}) in terms of the rotation variables.
In (\ref{pgo-lsp}), fixing $\mathbf{r}$ can yield an LSP for solving  $\mathbf{p}$:
\begin{equation}
	\mathbf{p}^{\star}=\arg \min _{\mathbf{p}}\|A \mathbf{p}+\tilde{D} \mathbf{r}\|^2_2	
\end{equation}
where $ \mathbf{p}^{\star}$ denotes the solution.
Since this problem is a simple unconstrained LSP, 
we can directly obtain the closed-form expression for  $ \mathbf{p}^{\star}$:
\begin{equation}\label{p-star}
	\mathbf{p}^{\star} = -A^{\dagger} \left(\tilde{D} \mathbf{r} \right) 
	+ c  \mathbf 1
\end{equation}
where $A^{\dagger} $ is the Moore–Penrose pseudo inverse of  $A$, and $c$ denotes the constant for global translation. 
Obviously, the above solution is not unique due to the inherent global translation term in the anchor-free form. However, since this term has no impact on the objective function and subsequent analysis, it can be simply neglected.
Substituting (\ref{p-star}) into (\ref{pgo-lsp}), we have 
\begin{equation}\label{pgo-only-r}
	\begin{array}{ll}
		\underset{\mathbf{r}}{\operatorname{min}} & \left\|P_A^{\perp}\left(\tilde{D} \mathbf{r}\right)\right\|^2_2+\left\|\tilde{C} \mathbf{r} \right\|^2_2 \\
		\text { s.t.} & \left|r_i\right|=1, \quad i=0, 1, \ldots, N-1 
	\end{array}
\end{equation}
where $ P_A^{\perp} = I-A A^{\dagger}$ is the orthogonal projector onto the orthogonal complement of range($A$).
Furthermore, by expanding the objective function into a standard quadratic form, (\ref{pgo-only-r}) can be rewritten as:
\begin{equation}\label{pgo-unit-form}
	\begin{array}{ll}
	\underset{\mathbf{r}}{\operatorname{min}} & f(\mathbf r) 
	=  \mathbf{r}^H \left(\tilde{D}^H P_A^{\perp} \tilde{D} +
	 \tilde{C}^H \tilde{C} \right) \mathbf{r} 
	 =  \mathbf{r}^H \tilde{\mathcal{Q}} \mathbf{r} 
	 \\
	\text { s.t.} & \left|r_i\right|=1, \quad i=0, 1, \ldots, N-1 
\end{array}
\end{equation}
where  $\tilde{\mathcal{Q}}$  can be referred to as the \textbf{pose graph matrix}, which is Hermitian and positive semidefinite. 
Obviously, problem (\ref{pgo-unit-form}) is equivalent to the general form (\ref{pgo-lsp}), 
and the properties of  $\tilde{\mathcal{Q}}$   determine the optimization landscape of  problem.

It is noted that both  (\ref{pgo-lsp}) and (\ref{pgo-unit-form}) represent operable  formulations of  PGO. However, (\ref{pgo-unit-form}) offers a more compact form with a lower-dimensional solution space, the subsequent analysis will 
be established on the base of (\ref{pgo-unit-form}).

\subsection{Gauge-normalized form of PGO}

To facilitate the theoretical analysis of PGO, we further apply gauge normalization to its  formulation with respect to the true solution.

For problem (\ref{pgo-unit-form}),
let its planted solution be
\begin{equation}
	\mathbf{r}^{\star}=\left[\begin{array}{llll}
		{r}_0^{\star} & r_1^{\star} & \cdots & r_{N-1}^{\star}
	\end{array}\right]^T, \quad
	{\Phi}^{\star} \doteq \operatorname{diag}\left(\mathbf{r}^{\star}\right)
\end{equation}
and
\begin{equation}
	\mathbf{r}^{\star}=\Phi^{\star} \mathbf{1} 
\end{equation}
Introduce the gauge-normalized variable:
\begin{equation}\label{u to r}
	\mathbf{u} = \left(\Phi^{\star}\right)^H \mathbf{r}  \quad \rightarrow \quad
	\mathbf{r}=\Phi^{\star} \mathbf{u}
\end{equation}
then the planted solution satisfies:
\begin{equation}
	\mathbf{u}^{\star}=\left(\Phi^{\star}\right)^H \mathbf{r}^{\star}=\mathbf{1}
\end{equation}
Considering the noiseless form of (\ref{pgo-unit-form}) and 
substituting (\ref{u to r}) into it, we have
\begin{equation}
	\begin{aligned}
		\mathbf{r}^H \mathcal{Q}^{\star} \mathbf{r} & =\mathbf{u}^H\left(\Phi^{\star}\right)^H\left(\left(D^{\star}\right)^H P_A^{\perp} {D}^{\star}+\left({C}^{\star}\right)^H {C}^{\star}\right) \Phi^{\star} \mathbf{u} \\
		& = \mathbf{u}^H\left(\left(D^{\star} \Phi^{\star}\right)^H P_A^{\perp} {D}^{\star}\Phi^{\star}+\left({C}^{\star}\Phi^{\star}\right)^H {C}^{\star}\Phi^{\star}\right) \mathbf{u} \\
		& =\mathbf{u}^H \left(\left(D^\star_{\mathrm{eq}}\right)^H P_A^{\perp} D^\star_{\mathrm{eq}}+
		\left(C^\star_{\mathrm{eq}}\right)^H C^\star_{\mathrm{eq}}\right) \mathbf{u} \\
		& =\mathbf{u}^H  L_{\mathrm{eq}}^{\star} \mathbf{u}
	\end{aligned}
\end{equation}
where $\mathcal{Q}^{\star} , D^{\star}, C^{\star} $
denote the noiseless  $\tilde{\mathcal{Q} }, \tilde D^{}, \tilde C^{} $,  respectively,
and we can obtain
\begin{equation}
	\mathcal{Q}^{\star}=\Phi^{\star} L^{\star}_{\mathrm{eq}}\left(\Phi^{\star}\right)^H
\end{equation}
It is noted that $ L_{\mathrm{eq}}^{\star}$ is 
Hermitian and positive semidefinite, and 
has the planted null vector $\mathbf{1}$.
Therefore,
the noisy pose graph matrix $\tilde{\mathcal{Q}}$ can be denoted as
\begin{equation} \label{eq-of-Q}
	\begin{aligned}
		\tilde{\mathcal{Q}} & = \mathcal{Q}^{\star} + \Delta \mathcal{Q} \\
		&= 
		\Phi^{\star} L^{\star}_{\mathrm{eq}}\left(\Phi^{\star}\right)^H+\Delta \mathcal{Q}
	\end{aligned}
\end{equation}
where $\Delta \mathcal{Q}$ is the Hermitian noise matrix.

It can be noted that the above gauge-normalized form incorporates the true solution and normalizes it to $\mathbf 1$, which serves as the foundation for the subsequent analysis of orbit radius and optimality.

Before formally presenting the theoretical results, we first outline the overall analysis framework of this paper.
First, we present the standardized formulation of the PGO problem (Eq. (\ref{pgo-unit-form})), which serves as the fundamental formulation for subsequent analysis. Then, the explicit upper bounds of the BOA are analyzed based on different norm forms 
(Theorem \ref{lemma-1-2-norm} and Theorem \ref{d-inf-lemma}). Next, we derive the dual certificate (Eq. (\ref{k=q-R})) for optimality via the KKT conditions, and the optimality conditions over the entire feasible set of the original problem are provided 
(Theorem \ref{lemma-rank-k-n-1}). Finally, based on the preceding derivations, the impact of noise is analyzed (Property \ref{ppt-1} and Property 
\ref{new-tilde-Q-r-bound}), and the norm upper bound of noise matrix satisfying the optimality conditions within the BOA is further established (Theorem \ref{final-theo}).

%

\section{ Scope of the Basin of Attraction }

For a practical PGO problem, the presence of measurement noise directly complicates the analysis and computation of the global optimum. If the entire solution space is considered, the inherent non-convexity introduces numerous local minima, rendering the analysis intractable. Since the objective function of the original problem is in a standard quadratic form, it can be intuitively observed that the global optimum resides within the region defined by $\mathbf{r}^H \tilde{\mathcal{Q}} \mathbf{r}  \leq \left(\mathbf{r}^{\star}\right)^H\tilde{\mathcal{Q}} \mathbf{r}^{\star}$. This region is referred to as the BOA. 
In this section, we will utilize the true solution as a reference to characterize 
the explicit scope  and geometric radius bounds of this BOA in terms of both the 2-norm and the  $\infty$-norm.
%
%
%

We first investigate the radius bound in the average perspective ($2$-norm), as stated in the following theorem.

\begin{mytheo}[Radius bound of BOA in $2$-norm]\label{lemma-1-2-norm}
	For problem  (\ref{pgo-unit-form}) with (\ref{eq-of-Q}),
	and any feasible solution $\mathbf{r}$ in BOA, i.e., satisfying 
	\begin{equation}\label{r-r*-obj-lemma1}
		\mathbf{r}^H \tilde{\mathcal{Q}} \mathbf{r}  \leq \left(\mathbf{r}^{\star}\right)^H
	 \tilde{\mathcal{Q}} \mathbf{r}^{\star}
	\end{equation}  
	 define the orbit radius as
	 \begin{equation}\label{def-d}
	 	d \doteq \min _{\theta \in \mathbb{R}}\left\|e^{{j} \theta} \mathbf{r}-\mathbf{r}^{\star}\right\|_2
	 \end{equation}
	where $\theta$ denotes the global phase difference, 
	 then 
	 \begin{equation}\label{lemma1-res}
	 	d \leq \frac{4 \sqrt{N}}{\gamma^{\star}}\|\Delta \mathcal Q\|
	 \end{equation}
	 where  $\gamma^{\star}$  is  the smallest positive eigenvalue  of 
	 $L^{\star}_{\mathrm{eq}}$, i.e., the spectral gap.
\end{mytheo}

\begin{proof}
	The proof consists of two parts: first, we use gauge normalization to convert the orbit radius into a distance from the vector $\mathbf 1$, allowing us to restrict the radius to orthogonal complement of $\mathbf 1$ via orthogonal decomposition; second, we apply a quadratic inequality of the pose graph matrix to bound the vectors in this subspace, ultimately yielding an upper bound on the orbit radius.
	
	Let $\hat{\mathbf{r}}=e^{\mathrm j \theta^{\star}} \mathbf{r}$, where 
	$\theta^{\star} $ attains the minimum in (\ref{def-d}).
	This indicates that the feasible solution in BOA and the planted solution remain consistent in terms of global phase, then
	\begin{equation*}
		d=\left\|\hat{\mathbf{r}}-\mathbf{r}^{\star}\right\|_2, \quad\|\hat{\mathbf{r}}\|_2=\left\|\mathbf{r}^{\star}\right\|_2, \quad \hat{\mathbf{r}}^H 
		\tilde{\mathcal{Q}} \hat{\mathbf{r}}=\mathbf{r}^H \tilde{\mathcal{Q}} \mathbf{r}
	\end{equation*}
	
	\textbf{Step 1:} Employing orthogonal projection to transform the orbit radius into the orthogonal complement of $\bf 1$.
	
	For planted solution matrix $\Phi^{\star}$, since it is unitary and  can preserve Euclidean norms, we have 
	\begin{equation*}
		\left\|{\Phi^{\star}}^H \mathbf{a}-{\Phi^{\star}}^H \mathbf{b}\right\|_2
		=\|\mathbf{a}-\mathbf{b}\|_2, \quad \forall \mathbf{a}, \mathbf{b} \in \mathbb{C}^N
	\end{equation*}
	Substituting $\Phi^{\star}$ into $d$, we can obtain 
	\begin{equation}\label{d-u-1}
		d=\left\|\hat{\mathbf{r}}-\mathbf{r}^{\star}\right\|_2=\left\|{\Phi^{\star}}^H \hat{\mathbf{r}}-{\Phi^{\star}}^H \mathbf{r}^{\star}\right\|_2
		=\|\mathbf{u}-\mathbf{1}\|_2
	\end{equation}
	Eq. (\ref{d-u-1}) shows that the Euclidean distance between the feasible solution in BOA and the planted solution is equivalent to the distance between the gauge-normalized variable and the all-ones vector.
	
	Hence, we can 
	decompose $\mathbf{u}$ in one-dimensional subspace $\text{span}\{\mathbf{1}\}$:
	\begin{equation}\label{u-xi-1}
		\mathbf{u}=\alpha \mathbf1+\bm{\xi}, \quad \bm{\xi} \perp \mathbf1, \quad \alpha=\frac{\mathbf{1}^H \mathbf{u}}{N} 
	\end{equation}
	According to the properties of orthogonal decomposition, we have
	\begin{equation}\label{xi-beq-u}
		\begin{aligned}
			\|\bm\xi\|_2^2 & =\|\mathbf u\|_2^2-|\alpha|^2\|\mathbf 1\|_2^2 \\
			& =N - \frac{\left|\mathbf 1^H \mathbf u\right|^2}{N} \\
			& = \frac{\left(N-\left|\mathbf 1^H \mathbf{u}\right|\right)\left(N+\left|\mathbf 1^H \mathbf{u}\right|\right)}{N} \\
			& \geq N-\left|\mathbf 1^H \mathbf{u}\right|
		\end{aligned}
	\end{equation}
	
	Expanding (\ref{d-u-1}), we can obtain
	\begin{equation*}
		\begin{aligned}
			d^2 & =(\mathbf u-\mathbf 1)^H(\mathbf u-\mathbf 1) \\
			& =\mathbf u^H \mathbf u+\mathbf 1^H \mathbf 1-\mathbf u^H \mathbf 1-\mathbf 1^H \mathbf u \\
			& =\|\mathbf u\|_2^2+\|\mathbf 1\|_2^2-2 \Re\left(\mathbf 1^H \mathbf u\right) \\
			& =2 N-2 \Re\left(\mathbf 1^H \mathbf u\right)
		\end{aligned}
	\end{equation*}
	Since the introduction of $\theta^{\star}$ has achieved phase alignment, we have 
	\begin{equation*}
		\begin{aligned}
		\mathbf{1}^H \mathbf u &=\sum_{i=0}^{N-1} u_i=e^{-j \theta^{\star}} \sum_{i=0}^{N-1} \overline{r_i^{\star}} r_i \\ 
		&=e^{-j \theta^{\star}}\left(\mathbf r^{\star}\right)^H \mathbf r = \left|\left(\mathbf r^{\star}\right)^H \mathbf r\right|
		\end{aligned}
	\end{equation*}
	Therefore, $\mathbf 1^H \mathbf u$ is a non-negative real number, and we have
	\begin{equation}\label{d2-2-N}
		d^2=2\left(N-\left|\mathbf 1^H \mathbf u\right|\right)
	\end{equation}
	Back to (\ref{xi-beq-u}), we can obtain
	\begin{equation}\label{xi-beg-1/2d}
		\|\bm\xi\|_2^2 \geq \frac{1}{2} d^2
	\end{equation}
	
	The above equation indicates that, through orthogonal decomposition, the orbit radius is bounded by the norm of $\bm \xi$ in the orthogonal complement.
	It now remains to derive the upper bound of $\|\bm\xi\|_2^2$.
	
	\textbf{Step 2:}  Establishing the upper bound of $\|\bm\xi\|_2^2$ by leveraging the quadratic form inequality of pose graph matrix.
	
	For pose graph matrix $\mathcal{Q}^{\star}$, it implies 
	\begin{equation}\label{rqr=ulu}
		\begin{aligned}
			\hat{\mathbf{r}}^H \mathcal{Q}^{\star} \hat{\mathbf{r}} & 
			=\mathbf{u}^H L_{\mathrm {eq}}^{\star} \mathbf u \\
			& =(\alpha \mathbf{1}+\bm\xi)^H L_{\mathrm {eq}}^{\star}(\alpha \mathbf{1}+\bm\xi) \\
			& =\bm\xi^H L^{\star}_{\mathrm {eq}} \bm\xi
		\end{aligned}
	\end{equation}
	Since $L^{\star}_{\mathrm {eq}}$ has the planted null vector $\mathbf{1}$, 
	for  any vector $ \bm\xi$ orthogonal to $\mathbf{1}$, it can give 
	\begin{equation*}
		\hat{\mathbf{r}}^H \mathcal{Q}^{\star} \hat{\mathbf{r}} =\bm\xi^H L^{\star}_{\mathrm {eq}} \bm\xi \geq \gamma^{\star}\|\bm\xi\|_2^2
	\end{equation*}
	
	Next for (\ref{r-r*-obj-lemma1}), we have
	\begin{equation*}
		\begin{aligned}
			\hat{\mathbf{r}}^H\left(\mathcal{Q}^{\star}+\Delta \mathcal Q\right) \hat{\mathbf{r}} & \leq \left(\mathbf{r}^{\star}\right)^H\left(\mathcal Q^{\star}+\Delta \mathcal Q\right) \mathbf{r}^{\star} \to \\
			\hat{\mathbf{r}}^H \mathcal Q^{\star} \hat{\mathbf{r}} & \leq \left(\mathbf{r}^{\star}\right)^H \Delta \mathcal Q \mathbf{r}^{\star}-\hat{\mathbf{r}}^H \Delta \mathcal Q \hat{\mathbf{r}}
		\end{aligned}
	\end{equation*}
	Furthermore, we have
	\begin{equation*}
		\begin{aligned}
			&\hat{\mathbf{r}}^H \Delta \mathcal Q \hat{\mathbf{r}}-\left(\mathbf{r}^{\star}\right)^H \Delta \mathcal Q \mathbf{r}^{\star}  \\
			& =\left(\hat{\mathbf{r}}-\mathbf{r}^{\star}\right)^H \Delta \mathcal Q \hat{\mathbf{r}}+\left(\mathbf{r}^{\star}\right)^H \Delta \mathcal Q\left(\hat{\mathbf{r}}-\mathbf{r}^{\star}\right) \\
			& =\Re\left\{\left(\hat{\mathbf{r}}-\mathbf{r}^{\star}\right)^H \Delta \mathcal Q\left(\hat{\mathbf{r}}+\mathbf{r}^{\star}\right)\right\}
		\end{aligned}
	\end{equation*}
	Therefore, it can give 
	\begin{equation}
		\label{d-dq-2-n}
		\begin{aligned}
			\left|\hat{\mathbf{r}}^H \Delta \mathcal Q \hat{\mathbf{r}}-\left(\mathbf{r}^{\star}\right)^H \Delta \mathcal Q \mathbf{r}^{\star}\right|
			& \leq\left\|\hat{\mathbf{r}}-\mathbf{r}^{\star}\right\|_2\|\Delta \mathcal Q\|\left\|\hat{\mathbf{r}}+\mathbf{r}^{\star}\right\|_2 \\
			& \leq d \cdot\|\Delta \mathcal Q\| \cdot 2 \sqrt{N} 
		\end{aligned}
	\end{equation}
	and 
	\begin{equation}\label{xi-gamma-d-frac}
		\|\bm\xi\|_2^2 \leq \frac{d \cdot\|\Delta \mathcal Q\| \cdot 2 \sqrt{N}}{\gamma^{\star}}
	\end{equation}
	
	As a result, combining (\ref{xi-beg-1/2d}), (\ref{xi-gamma-d-frac}), we can obtain
	\begin{equation*}
		\begin{aligned}
		\frac{\gamma^{\star}}{2} d^2 &\leq 2 \sqrt{N}\|\Delta \mathcal Q\| d \quad \to \\
		d & \leq \frac{4 \sqrt{N}}{\gamma^{\star}}\|\Delta \mathcal Q\|
		\end{aligned}
	\end{equation*}
	which leads to  (\ref{lemma1-res}).
\end{proof}

Theorem \ref{lemma-1-2-norm} analyzes the upper bound of  orbit radius in the form of  $2$-norm, indicating that the radius is bounded in terms of average distance.
Next, we will analyze the orbit radius from the perspective of   $\infty$-norm and provide its explicit  bound.

 \begin{mytheo}[Radius bound of BOA in $\infty$-norm]\label{d-inf-lemma}
 	Under the premise of Theorem \ref{lemma-1-2-norm}, define the infinity norm orbit radius as 
 	\begin{equation}\label{def-d-inf}
 		d_{\infty} \doteq \min _{\theta \in \mathbb{R}}\left\|e^{{j} \theta} \mathbf{r}-\mathbf{r}^{\star}\right\|_{\infty}
 	\end{equation}
 	Let 
 	\begin{equation}\label{q-dot-tau-max}
 		\begin{aligned}
 			&\mathbf{q}^{(i)}  \doteq \mathbf{s}^{(i)}-\frac{1}{N} \mathbf{1}, \\
 			&{\tau}^{(i)} \doteq (\mathbf{q}^{(i)})^H \left({L_{\mathrm{eq}}^{\star}}\right)^{\dagger}  (\mathbf{q}^{(i)}) \\
 			&   \tau_{\max } \doteq \max _{1 \leq i \leq N} \tau^{(i)}
 		\end{aligned}
 	\end{equation}
 where $\mathbf{s}^{(i)}$ denotes the $i$-th standard basis vector,
 then 
 \begin{equation}\label{d_inf-leq-2sq}
 	\begin{aligned}
 		d_\infty &\leq \sqrt{2 \sqrt{N} \tau_{\max }\|\Delta \mathcal Q\| d}+\frac{d^2}{2 N} \\
 		&\leq \frac{2 \sqrt{2} \sqrt{N}}{\gamma^{\star}}\|\Delta \mathcal  Q\|+\frac{8}{(\gamma^{\star})^2}\|\Delta \mathcal Q\|^2
 	\end{aligned}
 \end{equation}
  \end{mytheo}
 
 \begin{proof}
 	The proof proceeds analogously to Theorem \ref{lemma-1-2-norm}.
 	First, via orthogonal decomposition, the radius analysis reduces to analyzing $\bm \xi$,
 	which is orthogonal to $\mathbf 1$.
 	Then, we only need to investigate the upper bound of $\|\bm \xi\|_{\infty}$.
 	Since the infinity norm essentially involves analyzing each element of the vector, Eq. (\ref{q-dot-tau-max}) is introduced to extract each element of $\bm \xi$.
 	
 	\textbf{Step 1:} 
 	Converting 	the orbit radius into a bound analysis of $\|\bm \xi\|_{\infty}$ 
 	through orthogonal decomposition.
 	
 	Based on (\ref{u-xi-1}) (\ref{d2-2-N}),
 	we can obtain 
 	\begin{equation}\label{d2-2-N-2Na}
 		d^2 = 2 N-2 N \alpha  \to \alpha = \bm 1-\frac{d^2}{2 N}
 	\end{equation}
 	Hence, we have
 	\begin{equation}\label{u-1=xi-d2}
 		\mathbf u-\bm 1=\bm \xi-\frac{d^2}{2 N} \bm 1
 	\end{equation}
 	and 
 	\begin{equation}\label{d-inf=u-1}
 		d_\infty = \|\mathbf u-\bm 1\|_{\infty} \leq \|\bm \xi\|_{\infty}+\frac{d^2}{2 N}
 	\end{equation}
 	Consequently, bounding $d_\infty$ is reduced to bounding  $\|\bm \xi\|_{\infty}$.
 	
 	\textbf{Step 2:} Leveraging standard basis to extract each element of $\bm \xi$
 	 and employing  matrix decomposition to bound $\|\bm \xi\|_{\infty}$.
 	
 	Note that
 	\begin{equation*}
 		\mathbf{q}^{(i)}=\mathbf{s}^{(i)}-\frac{1}{N} \mathbf{1} \in \mathbf{1}^{\perp}
 	\end{equation*}
 	and since $\bm \xi \perp \mathbf{1}$, then we have 
 	\begin{equation*}
 		\xi_i={\mathbf s^{(i)}}^H \bm \xi=(\mathbf{s}^{(i)}-\frac{1}{N} \mathbf{1})^H \bm \xi={\mathbf{q}^{(i)}}^H \bm \xi
 	\end{equation*}
 	Because 
 	\begin{equation*}
 		\bm \xi = ({L_{\mathrm{eq}}^{\star}})^\dagger  {L_{\mathrm{eq}}^{\star}} \bm \xi
 	\end{equation*}
 	then 
 	\begin{equation*}
 		\begin{aligned}
 			\xi_i & =(\mathbf{q}^{(i)})^H 
 			({L_{\mathrm{eq}}^{\star}})^\dagger  {L_{\mathrm{eq}}^{\star}} \bm \xi \\
 			& = \left(\left(({L_{\mathrm{eq}}^{\star}})^\dagger\right)^{1 / 2} \mathbf{q}^{(i)}\right)^H\left(({L_{\mathrm{eq}}^{\star}})^{1 / 2} \bm \xi\right) 
 		\end{aligned}
 	\end{equation*}
 	Applying the Cauchy-Schwarz inequality to the above equation, we can obtain
 	\begin{equation*}
 		\begin{aligned}
 			\left|\xi_i\right|  &\leq 
 			\left\|\left(({L_{\mathrm{eq}}^{\star}})^\dagger\right)^{1 / 2} \mathbf{q}^{(i)}\right\|_2  \left\|({L_{\mathrm{eq}}^{\star}})^{1 / 2} \bm \xi\right\|_2 \\
 			& = \sqrt{(\mathbf q^{(i)})^H ({L_{\mathrm{eq}}^{\star}})^\dagger  \mathbf q^{(i)}} \sqrt{\bm \xi^H {L_{\mathrm{eq}}^{\star}} \bm \xi} \\
 			& = \sqrt{ \tau^{(i)}} \sqrt{\bm \xi^H {L_{\mathrm{eq}}^{\star}} \bm \xi}
 		\end{aligned}
 	\end{equation*}
 	Taking the maximum over above equation, we have
 	\begin{equation}\label{xi-inf-bd}
 		\|\bm \xi\|_{\infty} \leq \sqrt{\tau_{\max }} \sqrt{\bm \xi^H {L_{\mathrm{eq}}^{\star}} \bm \xi}
 	\end{equation}
 	Based on (\ref{rqr=ulu}) (\ref{d-dq-2-n}), we can get
 	\begin{equation}\label{xilxi=rqr}
 		\bm\xi^H L^{\star}_{\mathrm {eq}} \bm\xi  = \hat{\mathbf{r}}^H \mathcal{Q}^{\star} \hat{\mathbf{r}} \leq d \cdot\|\Delta \mathcal Q\| \cdot 2 \sqrt{N}
 	\end{equation}
 	Combining (\ref{xi-inf-bd}) and (\ref{xilxi=rqr}), it follows that
 	\begin{equation}\label{xi-inf-leq-taomax}
 		\|\bm \xi\|_{\infty} \leq \sqrt{\tau_{\max }} \sqrt{2 \sqrt{N}\|\Delta \mathcal Q\| d }
 	\end{equation}
 	Back to (\ref{d-inf=u-1}), we have 
 	\begin{equation*}
 		d_\infty \leq \sqrt{2 \sqrt{N} \tau_{\max }\|\Delta \mathcal Q\| d}+\frac{d^2}{2 N}
 	\end{equation*}
 	
 	Since 
 	\begin{equation*}
 		\begin{aligned}
 			\tau^{(i)}  & = (\mathbf{q}^{(i)})^H \left({L_{\mathrm{eq}}^{\star}}\right)^{\dagger}  (\mathbf{q}^{(i)})  \\ 
 			& \leq 
 			\left\|\left({L_{\mathrm{eq}}^{\star}}\right)^{\dagger} \right\| \|\mathbf{q}^{(i)}\|_2^2 \\
 			& \leq  \left\|\left({L_{\mathrm{eq}}^{\star}}\right)^{\dagger} \right\|
 			= \frac{1}{\gamma^{\star}}
 		\end{aligned}  \to  \quad \tau_{\max} \leq \frac{1}{\gamma^{\star}}
 	\end{equation*}
 	then 
 	\begin{equation*}
 		d_\infty \leq  \frac{2 \sqrt{2} \sqrt{N}}{\gamma^{\star}}\|\Delta \mathcal  Q\|+\frac{8}{(\gamma^{\star})^2}\|\Delta \mathcal Q\|^2
 	\end{equation*}
 	which implies (\ref{d_inf-leq-2sq}).
 \end{proof}
 
 Theorem \ref{d-inf-lemma} provides an upper bound for each individual element in the orbit radius vector.
 Combined with Theorem \ref{lemma-1-2-norm}, the aforementioned two theorems characterize the scope of BOA from different perspectives, providing geometric upper bounds for the region where the global optimum resides.

 Based on the above
analysis, we will investigate the global optimality of the
solution in the next context.

\section{Optimality Conditions for the Solution}

In this section, we will thoroughly analyze the conditions under which the feasible solution satisfies global optimality, in order to reveal the optimization landscape of the PGO problem itself. 
Since the original PGO problem is complex, we will conduct our analysis via its SDP relaxation.

Many existing works have already provided the SDP relaxation of  PGO; hence, this paper will not derive it in detail. 
For the convenience in analysis and exposition, we will first directly present the SDP relaxation formulation of original problem below.

\subsection{Semidefinite Programming Relaxation}

For original problem (\ref{pgo-unit-form}), define 
\begin{equation}
	X \doteq {\mathbf r} {\mathbf r}^H
\end{equation}
then  (\ref{pgo-unit-form}) can be rewritten as
\begin{equation}\label{sdp-have-rank}
		\begin{array}{ll}
		\underset{{X\in \mathbb{C}^{N \times N}}}{\operatorname{min}} & \mathbf{r}^H \tilde{\mathcal{Q}} \mathbf{r}
		=  \operatorname{Tr}(\tilde{\mathcal{Q}} X)
		\\
		\text { s.t.} & X \succeq 0, \quad  X_{ii} = 1, \quad  \operatorname{rank}(X)=1 
	\end{array}
\end{equation}
It can be seen that  $X$  is a Hermitian matrix.
Dropping the rank constraint in (\ref{sdp-have-rank}) yields the 
convex relaxation of  
(\ref{pgo-unit-form}):
	\begin{equation}\label{sdp-no-rank}
		\begin{array}{ll}
			\underset{{X\in \mathbb{C}^{N \times N}}}{\operatorname{min}} & \mathbf{r}^H \tilde{\mathcal{Q}} \mathbf{r}
			=  \operatorname{Tr}(\tilde{\mathcal{Q}} X)
			\\
			\text { s.t.} & X \succeq 0, \quad  X_{ii} = 1
		\end{array}
	\end{equation}
and it is an SDP problem.
Obviously, the original problem can be solved by problem (\ref{sdp-no-rank}) and verifying whether the solution has rank one. 
Since the tightness of the relaxation in (\ref{sdp-no-rank}) also reflects the optimization landscape of the original problem, the subsequent analysis will leverage this SDP formulation to assess the global optimality of  feasible solution.

\subsection{Karush-Kuhn-Tucker  Conditions}

The global optimizer of the SDP (\ref{sdp-no-rank})  can be characterized by the 
	KKT conditions, as follows.

\begin{myppt}\label{kkt-sdp}
	For problem (\ref{sdp-no-rank}), the feasible solution $X$ is a global optimizer if  and
	only if there exists a matrix $\hat K \in \mathbb{C}^{N \times N}$ such that:
	\begin{enumerate}
		\item $X_{ii}=1$,
		\item $X \succeq 0$,
		\item $\hat K X = \mathbf 0$,
		\item $\tilde {\mathcal Q} - \hat K $ is real diagonal,
		\item $\hat K \succeq 0$
	\end{enumerate}
	Furthermore, if $\operatorname{rank}(\hat K)=N-1 $, then $\operatorname{rank}(X)=1 $ and it is the unique global optimizer of (\ref{sdp-no-rank}), 
	while the  $\mathbf{r}$  that constitutes this  $X $ is  the global optimizer of the primal problem (\ref{pgo-unit-form}).
\end{myppt}

The Property \ref{kkt-sdp}  is the standard KKT condition and framework; therefore, this paper omits the detailed construction, and its proof can be found in  \cite{RN1176}.
It should be noted that the key point of KKT conditions lies in how to construct the matrix  $\hat K$, also known as the dual certificate. 
Next, we will construct the dual certificate  $\hat K$  based on the first-order optimality conditions of the primal problem (\ref{pgo-unit-form}).

\subsection{Dual Certificate}

The feasible region of primal problem (\ref{pgo-unit-form}) is essentially a manifold, which can be expressed as
\begin{equation}\label{M-N}
	\mathcal{M}^N \doteq\left\{\mathbf r \in \mathbb{C}^N:\left|r_i\right|=1,  i= 0,1,\cdots,N-1\right\}
\end{equation}
For each $r_i$, its  region is a complex circle and can be denoted as 
\begin{equation}
     \mathcal{M}_i=\left\{r_i \in \mathbb{C}: r_i^H r_i=1\right\}
\end{equation}
This circular region can be regarded as a submanifold, whose tangent plane is
\begin{equation}\label{T-M-i}
	\begin{aligned}
	T_{r_i} \mathcal{M}_i & =\left\{\dot{r_i} \in \mathbb{C}: \dot{r_i}^H r_i+r_i^H \dot{r_i}=0\right\} \\
	&=\{\dot{r_i} \in \mathbb{C}:\langle r_i, \dot{r_i}\rangle=0\}
	\end{aligned}
\end{equation}
Based on (\ref{T-M-i}), the tangent space of (\ref{M-N}) can be denoted as
\begin{equation}
	\begin{aligned}
		T_{\mathbf r} \mathcal{M}^N & =T_{r_0} \mathcal{M}_0 \times \cdots \times T_{r_{N-1}} \mathcal{M}_{N-1} \\
		& =\left\{\dot{\mathbf{r}} \in \mathbb{C}^N: \Re\left\{\operatorname{diagV}\left(\dot{\mathbf{r}}
		 {\mathbf{r}}^H\right)\right\}=\mathbf 0\right\}
	\end{aligned}
\end{equation}
and the orthogonal projector onto $T_{\mathbf r} \mathcal{M}^N$ is 
\begin{equation}
	\operatorname{Proj}_{\mathbf r}(\dot{\mathbf r})=\dot{\mathbf r }-\Re\left(\operatorname{diag}\left(\dot{\mathbf r} \mathbf r^H\right)\right) \mathbf r
\end{equation}

Based on the above derivation, in order to solve the first-order optimality conditions of  primal problem, we need to project the original Euclidean gradient onto  tangent space using the  projection operator to obtain the Riemannian gradient.

The Euclidean gradient of (\ref{pgo-unit-form}) is 
\begin{equation}
	\nabla f(\mathbf r)=2 \tilde{\mathcal{Q}} \mathbf r
\end{equation}
then the Riemannian gradient can be denoted as
\begin{equation}
	\begin{aligned}
		\operatorname{Rgrad} f(\mathbf r) & =\operatorname{Proj}_{\mathbf{r}}(\nabla f(\mathbf r)) \\
		& =2 \tilde{\mathcal{Q}} \mathbf r-2 \Re\left(\operatorname{diag}\left(\tilde{\mathcal{Q}} \mathbf r \mathbf r^H\right)\right) \mathbf r 
	\end{aligned}
\end{equation}
Hence, for each first-order critical point $\mathbf r$, it satisfies
\begin{equation}\label{rgrad-f}
	\operatorname{Rgrad} f(\mathbf r) = \mathbf 0  \to  \left(\tilde{\mathcal{Q}} - \Re\left(\operatorname{diag}\left(\tilde{\mathcal{Q}} \mathbf r \mathbf r^H\right)\right) \right)  \mathbf r = \mathbf 0
\end{equation}

Noting condition 4 in Property \ref{kkt-sdp} and combining it with  (\ref{rgrad-f}), 
we  can construct
\begin{equation}\label{k=q-R}
	 K = \tilde{\mathcal{Q}} - \Re\left(\operatorname{diag}\left(\tilde{\mathcal{Q}} \mathbf r \mathbf r^H\right)\right)
\end{equation}
then 
\begin{equation}
	\tilde{\mathcal{Q}} -  K = \Re\left(\operatorname{diag}\left(\tilde{\mathcal{Q}} \mathbf r \mathbf r^H\right)\right)
\end{equation}
which is diagonal.
Therefore, we can let  $K$  be the dual certificate, 
which directly satisfies conditions 3) and  4) in  Property \ref{kkt-sdp}.
Since the SDP solution  $X$  is constructed from  $\mathbf{r} \mathbf{r}^H$, when $ \mathbf r$  is a first-order critical point, it can satisfy conditions 1) and 2).
The above shows that the constructed  $K$  and  $X$  are valid. 
Next, we will establish the uniqueness of  $K$  and show that the necessary and sufficient conditions for $ X$  to be  an optimal solution to  SDP is that   $K $ is positive semidefinite.

\begin{mytheo}[Dual certificate for global optimality]\label{k=rq-0}
	Let $X$ be a feasible point of the SDP problem (\ref{sdp-no-rank}), then
	$X$ is optimal 
	if and only if 
	$K = \tilde{\mathcal{Q}} - \Re\left(\operatorname{diag}\left(\tilde{\mathcal{Q}} \mathbf r \mathbf r^H\right)\right) \succeq 0$ .
	 In particular, there is no other dual certificate associated with $X$.
\end{mytheo}

\begin{proof}
    For the sufficient condition, since  $X_{ii} = 1$,
    we have
    \begin{equation*}
    	\begin{aligned}
    		\operatorname{Tr}(K X) & =\operatorname{Tr}\left(\tilde{\mathcal{Q}} X\right)-\operatorname{Tr}\left( \Re\left(\operatorname{diag}\left(\tilde{\mathcal{Q}} \mathbf r \mathbf r^H\right) \right)X\right) \\
    		& =\operatorname{Tr}\left(\tilde{\mathcal{Q}} X\right)-\sum_{i=1}^N \Re\left({\tilde{\mathcal{Q}} X}\right)_{ii} X_{i i} \\
    		& =\operatorname{Tr}\left(\tilde{\mathcal{Q}} X\right)-\Re\left(\operatorname{Tr}\left(\tilde{\mathcal{Q}} X\right)\right) \\
    		& = 0
    	\end{aligned}
    \end{equation*}
    Due to $K, X \succeq 0$, this  implies $KX = 0$.
    By construction, $\tilde{\mathcal{Q}}-K$ is diagonal, thus $X, K$ satisfy all KKT conditions in Property
    \ref{kkt-sdp}, then $X$ is optimal for (\ref{sdp-no-rank}).
    
    For the  necessary condition, assume $X$ is optimal, then there 
     exists a dual certificate $\hat K \succeq 0 $ such that
     \begin{equation*}
     	\hat{K} X=\mathbf 0, \quad \tilde{\mathcal{Q}} - \hat{K}=\operatorname{diag}({\mathbf y})
     \end{equation*}
     where  ${\mathbf y}$ is real vector.
     Thus, we have
     \begin{equation}\label{qx-diag-y-x}
     	\tilde{\mathcal{Q}} X=\operatorname{diag}({\mathbf y}) X
     \end{equation}
     Taking diagonal entries and using $X_{ii} = 1 $ gives
     \begin{equation*}
     	(\tilde{\mathcal{Q}} X)_{i i}={ y}_i = \Re\left(\left(\tilde{\mathcal{Q}} X\right)_{i i}\right)
     \end{equation*}
     Therefore, $\hat K$ can be written as
     \begin{equation*}
     	\hat K = \tilde{\mathcal{Q}} - \Re\left(\operatorname{diag}\left(\tilde{\mathcal{Q}} \mathbf r \mathbf r^H\right)\right) = K
     \end{equation*}
     which is positive semidefinite  and  also proves that no other certificate associated with $X$ can exist.
\end{proof}


The above discussion and Theorem \ref{k=rq-0} establish the relationship between the dual certificate and the optimal solution, showing that the dual certificate can serve as a sole criterion for verifying the optimality of an SDP solution. 
However, they do not yet analyze under what condition the rank of the dual certificate is  $N - 1$ , which determines the nature of the solution in primal problem. 
Next, we will  provide a detailed analysis of this issue.

\subsection{Optimality Conditions}

To analyze the optimality conditions for the primal problem,
we need to guarantee the  dual certificate $K$ is positive semidefinite 
with  its rank being  $N - 1$.
Regarding this, we have the following theorem.


\begin{mytheo}[Optimality condition over the entire feasible set]\label{lemma-rank-k-n-1}
	For PGO problem (\ref{pgo-unit-form}),
	let feasible solution $\mathbf r$ be a first-order critical point and 
	satisfy 
	\begin{equation}\label{gamma-1-d/n}
		\gamma^{\star}\left(1-\frac{d^2}{N}\right)>\|\Delta \mathcal Q\|+\|\mathbf y\|_{\infty}
	\end{equation}
	where $y_i=\overline{{{r}}_i}(\tilde{\mathcal Q} {\mathbf{r}})_i$, 
	then $\operatorname{rank}( K)=N-1 $.
	Consequently,  $X = \mathbf r {\mathbf r}^H$ is the unique global optimizer of SDP (\ref{sdp-no-rank}) and $\mathbf r$ is the  global optimizer of the primal problem (\ref{pgo-unit-form}).
\end{mytheo}

\begin{proof}
	The proof relies on analyzing the orthogonal complement of $\bf r$ to bound the rank of 
	$K$. Before delving into the proof details, we first provide some additional remarks on 
	$\mathbf y$.
	
	The definition of  $\mathbf y$  in the theorem is consistent with that in (\ref{qx-diag-y-x}). Since  $\mathbf r $ is a first-order critical point and $K\mathbf r = \mathbf 0$, we can obtain 
	\begin{equation*}
		\Re\left(\overline{{{r}}_i}(\tilde{\mathcal Q} {\mathbf{r}})_i\right) r_i= 		
		(\tilde{\mathcal{Q}}\mathbf{r})_i
	\end{equation*}
	Thus, $\overline{{{r}}_i}(\tilde{\mathcal Q} {\mathbf{r}})_i $
	is real and 
	\begin{equation*}
		y_i=\overline{{{r}}_i}(\tilde{\mathcal Q} {\mathbf{r}})_i = 
		\Re\left(\overline{{{r}}_i}(\tilde{\mathcal Q} {\mathbf{r}})_i\right) 
	\end{equation*}
	
	To analyze the rank of  $K$ , 
	we can consider the vector $\mathbf v$ in the orthogonal complement $\mathbf r^{\perp}$ of  $\mathbf r$.
For $\mathbf v \in \mathbf r^{\perp}$,  the quadratic form of  $K$ can be written as:
	\begin{equation}\label{vkv-ori}
		\mathbf v^H K \mathbf v=\mathbf v^H \mathcal Q^{\star} \mathbf v+ \mathbf v^H \Delta \mathcal Q \mathbf v-\sum_{i=1}^N y_i\left|v_i\right|^2
	\end{equation}
	In order for the rank of  $K $ to be  $N - 1$ , we need  $\mathbf v^H K \mathbf v > 0$.
	Therefore, we need to analyze the bound of each item in  (\ref{vkv-ori}).
	
	\textbf{Step 1: } Obtaining the bound of $\mathbf v^H \mathcal Q^{\star} \mathbf v$ through normalization and orthogonal projection.
	
	Define the gauge-normalized vector of $\mathbf v$ as
	\begin{equation*}
		\bm{\xi}_{v} \doteq (\Phi^{\star})^H \mathbf v
	\end{equation*}
	then $\|\bm \xi_v\|_2=\|\mathbf v\|_2$.
	Let 
	\begin{equation*}
		\alpha_v = \frac{1}{N} \mathbf 1^H \bm \xi_v=\frac{1}{N} (\mathbf{r}^{\star})^H \mathbf v
	\end{equation*}
	Since $\mathbf v^H \mathbf r=0 \text { and } \mathbf v^H \hat{\mathbf r}=0$,
	we have 
	\begin{equation*}
		(\mathbf{r}^{\star})^H \mathbf v=\left(\mathbf{r}^{\star}-\hat{\mathbf r}\right)^H \mathbf v
	\end{equation*}
	and then
	\begin{equation*}
		|\alpha_v| \leq \frac{\left\|\mathbf{r}^{\star}-\hat{\mathbf r}\right\|_2\|\mathbf v\|_2}{N}
		=\frac{d\|\mathbf v\|_2}{N}
	\end{equation*}
	Applying the orthogonal projection operator $P_{{\mathbf 1}}^{\perp}$ onto the orthogonal complement of $ \mathbf{1}$  to  $\bm \xi_v $, 
	we can have
	\begin{equation*}
		\begin{aligned}
			\left\|P_{\mathbf 1}^{\perp} \bm \xi_v\right\|_2^2 & =\|\bm \xi_v-\alpha_v \mathbf 1\|_2^2 \\
			& =\|\bm \xi_v\|_2^2-N|\alpha_v|^2 \\
			& \geq\left(1-\frac{d^2}{N}\right)\|\mathbf v\|_2^2
		\end{aligned}
	\end{equation*}
	Since $L_{\rm{eq}}^{\star} \succeq 0 $ and  $ L_{\rm{eq}}^{\star} \mathbf 1= \mathbf 0$,
	it follows that
	\begin{equation*}
		\begin{aligned}
			\mathbf v^H \mathcal Q^{\star} \mathbf v & =\bm \xi_v^H L_{\rm{eq}}^{\star} \bm \xi_v \\
			& =\left(P_{1}^{\perp} \bm \xi_v\right)^H L_{\rm{eq}}^{\star}\left(P_{1}^{\perp} \bm \xi_v\right) \\
			& \geq \gamma^{\star}\left\|P_{1}^{\perp} \bm \xi_v\right\|_2^2 \\
			& \geq \gamma^{\star}\left(1-\frac{d^2}{N}\right)\|\mathbf v\|_2^2
		\end{aligned}
	\end{equation*}
	
	\textbf{Step 2:}  Incorporating the bounds of the remaining items in (\ref{vkv-ori}) can obtain the theorem.
	
	Due to 
	\begin{equation*}
		\left|\mathbf v^H \Delta \mathcal Q \mathbf v\right| \leq\|\Delta \mathcal Q\|\|\mathbf v\|_2^2
	\end{equation*} 
	and 
	\begin{equation*}
	     \left|\sum_{i=1}^N y_i \left|v_i\right|^2\right|  \leq\|\mathbf y\|_{\infty}\|\mathbf v\|_2^2
	\end{equation*}
	then we can obtain
	\begin{equation*}
		\mathbf v^H K \mathbf v \geq\left(\gamma^{\star}\left(1-\frac{d^2}{N}\right)-\|\Delta \mathcal Q\|-\|\mathbf y\|_{\infty}\right)\|\mathbf v\|_2^2
	\end{equation*}
	Therefore, for $\mathbf v^H K \mathbf v > 0$, we have
	\begin{equation*}
		\gamma^{\star}\left(1-\frac{d^2}{N}\right)>\|\Delta \mathcal Q\|+\|\mathbf y\|_{\infty}
	\end{equation*}
	which implies Eq. (\ref{gamma-1-d/n}) and Theorem \ref{lemma-rank-k-n-1}.
\end{proof}

It should be noted that the optimality condition described in Theorem \ref{lemma-rank-k-n-1} imposes no restrictions on the region of the feasible solution, as it targets the first-order critical points within the entire feasible set. 
In Eq. (\ref{gamma-1-d/n}), the left-hand side (LHS) is free of noise terms, whereas the right-hand side (RHS) is subject to noise, which essentially acts as an implicit condition constraining the noise. 
The preceding sections have indicated that the optimal solution typically resides within the BOA.
Therefore, in the following section,  we will investigate the optimality conditions within the BOA, and derive the  explicit upper bound for the noise matrix.



\section{Optimality Conditions in BOA}

In this section, we  will investigate the optimality conditions within the BOA and derives more explicit and intuitive conditions based on Theorem \ref{lemma-rank-k-n-1}.
It can be observed that the form of Eq. (\ref{gamma-1-d/n}) in Theorem \ref{lemma-rank-k-n-1} is relatively compact, and more detailed results can be derived based on the bounds for  $d$  and  $\mathbf y $.

Since the bound on the BOA radius $d $ has been established, it only remains to derive the bound on $\mathbf y $ , which relies on the following two properties.

\begin{myppt}\label{ppt-1}
	Under the premise of Theorem \ref{lemma-1-2-norm}, assuming that the planted solution satisfies
	\begin{equation}
		\left\|\Delta \mathcal Q \mathbf{r}^{\star}\right\|_{\infty} \leq \eta_{\infty}
	\end{equation}
	then 
	\begin{equation}\label{deltaQ-r-est-inf-norm}
		\|\Delta \mathcal Q \mathbf{r}\|_{\infty} \leq \eta_{\infty}+\|\Delta \mathcal Q\| d \leq \eta_{\infty}+\frac{4 \sqrt{N}}{\gamma^{\star}}\|\Delta \mathcal Q\|^2
	\end{equation}
\end{myppt}

\begin{proof}
	As shown in the proof of Theorem \ref{lemma-1-2-norm}, we have 
	$\hat{\mathbf{r}}=e^{j \theta^{\star}} \mathbf{r}$. 
	Since multiplying by a global phase does not change the infinity norm of a vector, we have
	\begin{equation}
		\|\Delta \mathcal Q \mathbf{r}\|_{\infty}=\|\Delta \mathcal Q \hat{\mathbf{r}}\|_{\infty}
	\end{equation}
	Let 
	\begin{equation*}
		\Delta \mathcal Q \hat{\mathbf{r}}=\Delta \mathcal Q \mathbf{r}^{\star}+\Delta \mathcal Q\left(\hat{\mathbf{r}}-\mathbf{r}^{\star}\right)
	\end{equation*}
	then we can obtain 
	\begin{equation}\label{delta-Q-r}
		\begin{aligned}
			\|\Delta \mathcal Q \mathbf{r}\|_{\infty} & =\|\Delta \mathcal Q \hat{\mathbf{r}}\|_{\infty} \\
			& \leq\left\|\Delta \mathcal Q \mathbf{r}^{\star}\right\|_{\infty}+\left\|\Delta \mathcal Q\left(\hat{\mathbf{r}}-\mathbf{r}^{\star}\right)\right\|_{\infty} \\
			& \leq \eta_{\infty}+\|\Delta \mathcal Q\|\left\|\hat{\mathbf{r}}-\mathbf{r}^{\star}\right\|_2 \\
			& \leq \eta_{\infty}+\frac{4 \sqrt{N}}{\gamma^{\star}}\|\Delta \mathcal Q\|^2
		\end{aligned}
	\end{equation}
	which leads to Property \ref{ppt-1}.
\end{proof}

Property \ref{ppt-1} provides an upper bound on the effect of the noise matrix acting on the feasible solution in BOA. Based on this, we can further analyze the bound of the pose graph matrix acting, as shown below.

\begin{myppt}\label{new-tilde-Q-r-bound}
 	Let
 	\begin{equation}
 		\delta_{\max } \doteq \max _{1 \leq i \leq N}\left(L_{e q}^{\star}\right)_{i i}
 	\end{equation}
 	then 
 	\begin{equation}\label{q-r-inf-prope}
 		\begin{aligned}
 		\|\tilde{\mathcal Q} \mathbf r\|_{\infty}  \leq \eta_{\infty}+\sqrt{\frac{8 N \delta_{\max }}{\gamma^{\star}}}\|\Delta \mathcal Q\| 
 		+
\frac{4 \sqrt{N}}{\gamma^{\star}} \|\Delta \mathcal Q\|^2
 		\end{aligned}
 	\end{equation}
 \end{myppt}
 
 \begin{proof}
 	Based on the triangle inequality of norm, we have 
 	\begin{equation}\label{tilde-Q-r-leq-inf}
 		\|\tilde{\mathcal Q} \mathbf r\|_{\infty}
 		 \leq\left\|\mathcal Q^{\star} {\mathbf r}\right\|_{\infty}
 		+\|\Delta \mathcal Q {\mathbf r}\|_{\infty} 
 	\end{equation}
 	Utilizing the decomposition in (\ref{d2-2-N-2Na}) (\ref{u-1=xi-d2}), we can obtain
 	\begin{equation*}
 		{\mathbf r}=\Phi^{\star} \mathbf u=\alpha \mathbf r^{\star}+\Phi^{\star} \bm \xi
 	\end{equation*}
and 
 	\begin{equation*}
 		\begin{aligned}
 		& \mathcal Q^{\star} {\mathbf r}=\Phi^{\star} L_{\text{eq}}^{\star}\left(\Phi^{\star}\right)^H {\mathbf r}=\Phi^{\star} L_{\text{eq}}^{\star} \mathbf u=\Phi^{\star} L_{\text{eq}}^{\star} \bm \xi \\
 		& \to \quad   \left\|\mathcal Q^{\star} {\mathbf r}\right\|_{\infty}=\left\|L_{\text{eq}}^{\star} \bm \xi\right\|_{\infty}
 		\end{aligned}
 	\end{equation*}
 	then (\ref{tilde-Q-r-leq-inf}) can be rewritten as 
 	\begin{equation}\label{q-r-inf-leq}
 		\|\tilde{\mathcal Q} \mathbf r\|_{\infty} \leq \left\|L_{\text{eq}}^{\star} \bm \xi\right\|_{\infty}+
\|\Delta \mathcal Q \hat{\mathbf r}\|_{\infty} 
 	\end{equation}
 	Next, 
 	based on Property \ref{ppt-1},
 	we only need to determine the bounds  of 
 	$\left\|L_{\text{eq}}^{\star} \bm \xi\right\|_{\infty} $. 
 	
 	For any index $i$, we have 
 	\begin{equation*}
 		\left|\left(L_{\text{eq}}^{\star} \bm \xi\right)_i\right|=\left|(\mathbf{s}^{(i)})^H L_{\text{eq}}^{\star} \bm \xi\right|
 	\end{equation*}
 	Applying the Cauchy-Schwarz inequality, we can obtain 
 	\begin{equation*}
 		\left|(\mathbf{s}^{(i)})^H L_{\text{eq}}^{\star} \bm \xi\right| \leq \sqrt{(\mathbf{s}^{(i)})^H L_{e q}^{\star} \mathbf{s}^{(i)}} \sqrt{\bm \xi^H L_{e q}^{\star} \bm \xi}
 	\end{equation*}
 	then 
 	\begin{equation*}
 		\left\|L_{{\text{eq}}}^{\star} \bm \xi\right\|_{\infty} \leq \sqrt{\delta_{\max }} \sqrt{\bm \xi^H L_{\text{eq}}^{\star} \bm \xi}
 	\end{equation*}
 	and based on (\ref{xilxi=rqr}), we can further  have 
 	\begin{equation}\label{L-eq-xi-inf}
 		\left\|L_{\text{eq}}^{\star} \bm \xi\right\|_{\infty} \leq \sqrt{\frac{8 N \delta_{\max }}{\gamma^{\star}}}\|\Delta \mathcal Q\|
 	\end{equation}
 	
 	
 	Therefore, based on (\ref{q-r-inf-leq}) (\ref{L-eq-xi-inf}) , we can obtain 
 	\begin{equation*}
 		\begin{aligned}
 			\|\tilde{\mathcal Q} \mathbf r\|_{\infty}  \leq \eta_{\infty}+\sqrt{\frac{8 N \delta_{\max }}{\gamma^{\star}}}\|\Delta \mathcal Q\| 
 			+
\frac{4 \sqrt{N}}{\gamma^{\star}}\|\Delta \mathcal Q\|^2
 		\end{aligned}
 	\end{equation*}
 	which implies (\ref{q-r-inf-prope}).
 \end{proof}
 
Properties \ref{ppt-1} and \ref{new-tilde-Q-r-bound} analyze the effects of the noise matrix and the pose graph matrix, respectively. 
Based on the above analysis, we can
further obtain the bound of $\mathbf y$.

Since
\begin{equation}
	\left| y_i\right|=\left|\overline{{r}_i}(\tilde{\mathcal{Q}}  {\mathbf{r}})_i\right| \leq \left|(\tilde{\mathcal{Q}} {\mathbf{r}})_{i}\right|
\end{equation}
then
we can directly obtain
\begin{equation}\label{y-q-r-inf-lemma}
	\|{\mathbf y}\|_{\infty}  \leq \|\tilde{\mathcal{Q}} \mathbf{r}\|_{\infty}
\end{equation}

Furthermore, we can substitute the relevant bounds into (\ref{gamma-1-d/n})  to obtain a more explicit form.
Based on  Theorem \ref{lemma-1-2-norm}, we can obtain 
\begin{equation}\label{d-2-N-delta-Q}
	\frac{d^2}{N} \leq \frac{16}{(\gamma^{\star})^2}\|\Delta \mathcal Q\|^2
\end{equation}
Therefore, combining (\ref{gamma-1-d/n}) 
(\ref{y-q-r-inf-lemma})
(\ref{d-2-N-delta-Q}),
it follows that 
\begin{equation}\label{delta-Q-mix-bound}
	\begin{aligned}
		&\gamma^{\star}\left(1-\frac{16}{(\gamma^{\star})^2}\|\Delta \mathcal Q\|^2\right) >	\\
		&\|\Delta \mathcal Q\|+
	\eta_{\infty}+\sqrt{\frac{8 N \delta_{\max }}{\gamma^{\star}}}\|\Delta \mathcal Q\|_2 
	+ \frac{4 \sqrt{N}}{\gamma^{\star}} \|\Delta \mathcal Q\|_2^2
	\end{aligned}
\end{equation}
which can also lead to Theorem \ref{lemma-rank-k-n-1}.
It can be observed that Eq. (\ref{delta-Q-mix-bound}) primarily imposes a constraint on the norm of  noise matrix  $\Delta \mathcal Q$. 
Therefore, through simple algebraic manipulation, we can obtain an upper bound expression for $\|\Delta \mathcal Q\|$:
\begin{equation}\label{delta-q-norm-final}
		\|\Delta \mathcal Q\| < \\ \frac{-T_1+\sqrt{T_1^2+4T_2 \left(\gamma^{\star}-\eta_{\infty}\right)}}{2T_2}
\end{equation}
where $T_1 =1+\sqrt{\frac{8 N \delta_{\text {max }}}{\gamma^{\star}}}$ and 
$T_2  =\frac{16+4 \sqrt{N}}{\gamma^{\star}}$.

\begin{myrek}
It should be emphasized that, as a quadratic inequality in terms of  $\|\Delta \mathcal Q\| $, the validity of (\ref{delta-Q-mix-bound}) and (\ref{delta-q-norm-final}) implicitly requires a condition:  $\eta_{\infty} < \gamma^\star$. 
If this condition is not satisfied, the fundamental constraint $ \|\Delta \mathcal Q\|  \geq 0$  cannot be met. Hence,  $\eta_{\infty} < \gamma^\star$  is an inherent requirement.

Returning to Property \ref{ppt-1},  $\eta_{\infty}$ is the assumed upper bound on the influence of the noise matrix $ \Delta \mathcal Q$ on the planted solution $\mathbf r^\star $. It should be noted that this assumption is reasonable, because if the effect of the noise were unbounded, the problem would lose practical meaning. 
However, Property \ref{ppt-1} does not explicitly specify a concrete upper bound. 
Obviously, above condition $\eta_{\infty} < \gamma^\star$ actually serves as an upper bound for  $\eta_{\infty}$, thereby quantifying the effective impact of the noise matrix.
\end{myrek}

Based on the above analysis and discussion, we can restate Theorem \ref{lemma-rank-k-n-1} in a more detailed and direct form, as presented in the following theorem.
\begin{mytheo}[Optimality condition within the BOA]\label{final-theo}
	For PGO problem (\ref{pgo-unit-form}),
	let feasible solution $\mathbf r$ be a first-order critical point and 
	satisfy 
	\begin{equation}\label{bound-of-delta-Q-norm}
		\|\Delta \mathcal Q\| < \\ \frac{-T_1+\sqrt{T_1^2+4T_2 \left(\gamma^{\star}-\eta_{\infty}\right)}}{2T_2}
	\end{equation}
	where $T_1 = 1+\sqrt{\frac{8 N \delta_{\text {max }}}{\gamma^{\star}}}$,
	$T_2 = \frac{16+4 \sqrt{N}}{\gamma^{\star}}$ and  $\eta_{\infty} < \gamma^\star$.
	If $\mathbf r$ in BOA, i.e.,	\begin{equation*}
		\mathbf{r}^H \tilde{\mathcal{Q}} \mathbf{r}  \leq \left(\mathbf{r}^{\star}\right)^H
		\tilde{\mathcal{Q}} \mathbf{r}^{\star}
	\end{equation*}
	then $\mathbf r$ is the  global optimizer of the primal problem (\ref{pgo-unit-form}).
\end{mytheo}


\begin{myrek}
	The conditions in Theorem \ref{lemma-rank-k-n-1} and Theorem \ref{final-theo} are essentially designed to verify whether the feasible solution is globally optimal; however, they exhibit distinct differences. 

	Theorem \ref{lemma-rank-k-n-1} can be served as a posteriori condition, meaning that an feasible solution must first be obtained. If the orbit radius of this solution is sufficiently small and satisfies the proposed inequality, its global optimality is guaranteed. 
	In other words, Theorem \ref{lemma-rank-k-n-1} indicates that for the PGO problem, a good solution is highly likely to be globally optimal.
	Meanwhile, Theorem \ref{lemma-rank-k-n-1} is a more general condition that is not restricted to the BOA but applies to the entire feasible set.

	In contrast, Theorem \ref{final-theo} can be regarded as an a priori condition. 
	It allows for the verification of global optimality for the PGO problem based solely on the noise matrix satisfying the proposed bound, without the need to compute the feasible solution in advance. 
	Although Theorem \ref{final-theo} still requires the feasible solution to be a first-order critical point within the BOA, these conditions are easy to satisfy using general numerical methods. 
	Therefore, Theorem \ref{final-theo} indicates that when the noise matrix satisfies the proposed bound, the global optimum of the problem can be readily obtained.
	
	In summary, both Theorem \ref{lemma-rank-k-n-1} and Theorem \ref{final-theo} 
	provide the analytical conditions under which the PGO problem achieves global optimality, essentially indicating that the problem possesses a favorable optimization landscape—namely, under certain conditions, the point satisfying the first-order critical condition is actually the global optimum of the problem.
\end{myrek}

\begin{table}[h]
	\centering  
	\renewcommand\arraystretch{1.2} 
	\caption{Parameters  and simulation results for 4 robot nodes}
	\begin{tabular}{ccc}
		\toprule  
		Parameter & Description &  Value \\
		\midrule 
		$ \gamma^{\star} $   &   Spectral gap     &  $ 2.229 $  \\ 
		$d $   &    Orbit radius ($2$-norm)     &  $ 0.065 $  \\ 
		Eq.  (\ref{lemma1-res})    &   Bound of $d$  &  $  0.615   $  \\
		$d_{\infty}$  & Orbit radius ($\infty$-norm)    &  $0.048$\\
		Eq. (\ref{d_inf-leq-2sq} )   & Bound of  $d_{\infty}$ &  $0.482$ \\
		$ \eta_{\infty} (\left\|\Delta \mathcal Q \mathbf{r}^{\star}\right\|_{\infty})  $  
		& Bound of noise acting on truth &  $0.229$ \\
		$\|\Delta \mathcal Q \mathbf{r}\|_{\infty}$ &  Noise acting on feasible point   &  $0.232$ \\ 
		Eq. (\ref{deltaQ-r-est-inf-norm}) & Bound of $\|\Delta \mathcal Q \mathbf{r}\|_{\infty}$     & $0.335$ \\
		$\|\tilde{\mathcal Q} \mathbf r\|_{\infty}$ & Pose graph acting on feasible point  & $0.053$ \\
		Eq. (\ref{q-r-inf-prope})  & Bound of $\|\tilde{\mathcal Q} \mathbf r\|_{\infty}$ & $1.682$  \\	
		LHS of Eq. (\ref{gamma-1-d/n}) & Left hand side of Eq. (\ref{gamma-1-d/n}) &   $2.227$   \\
		RHS of Eq. (\ref{gamma-1-d/n}) & Right hand side of Eq. (\ref{gamma-1-d/n}) &  $0.224$   \\
		$\|\Delta \mathcal Q\|$ & $2$-norm of noise matrix  &  $0.171$  \\
		Eq. (\ref{bound-of-delta-Q-norm}) & Bound of noise matrix   &  $0.184$\\
		\bottomrule 
	\end{tabular}
\label{4-nodes-tab_simulation}
\end{table}

\begin{figure}[h] 
\centering 
\includegraphics[width=0.45\textwidth]{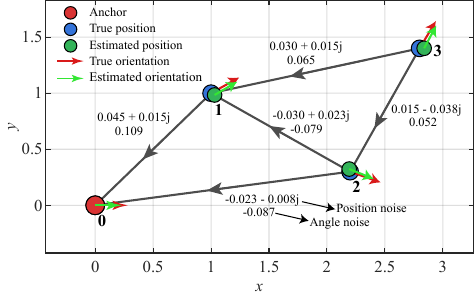} 
\caption{PGO for 4 robot nodes.} 
\label{fig-4-node} 
\end{figure}

\section{Simulation Results}

In this section, we will present numerical results to more intuitively illustrate and validate the theoretical results derived in this paper.
For brevity without loss of generality, we conduct simulations using the scenario of  relative localization problem among multiple robots, where the robot nodes and  inter-robot observations form a pose graph.

To reduce the impact of the algorithm itself on the problem, we employ only the classic and general Gauss-Newton method to solve the PGO problem.
To ensure the algorithm obtains the optimal solution, multiple sets of initial values are employed to drive distinct iterations, thereby helping the algorithm avoid local optima. 
For the final estimated solution, first-order critical point verification is conducted. 
Additionally, the positive semidefiniteness and the rank condition (rank = $N-1$) of  dual certificate $K$ are checked to validate the optimality of the estimate solution.

It is  noted that all the data below use the International System of Units (SI), while the simulation platform is based on Windows 11 (Intel i5-12600KF) and MATLAB 2024.

\subsection{Case A: Simple Illustration via An Instance}

In this case, we first utilize a simple and intuitive example of multi-robot relative localization to briefly illustrate the theorems and properties proposed in this paper.
We consider a relative localization problem with 4 robot nodes, as shown in  Fig. \ref{fig-4-node}, which  includes the true solution, estimated solution, and observation noise.
The relevant parameters and simulation results are also presented in  Table \ref{4-nodes-tab_simulation}.

It should be noted  that in Table \ref{4-nodes-tab_simulation}, since 
$\eta_{\infty}$ is the assumed upper bound for $\left\|\Delta \mathcal Q \mathbf{r}^{\star}\right\|_{\infty}$, we directly set them equal for convenience.
Through the verification of dual certificate, the estimated solution obtained by the Gauss-Newton method is a global optimum. 
As shown in Table \ref{4-nodes-tab_simulation}, the orbit radius (in both the $2$-norm and the $\infty$-norm) between the estimated solution and the true solution fall within the proposed upper bounds. 
For Theorem \ref{lemma-rank-k-n-1}, the data in table indicate that the LHS of Eq. (\ref{gamma-1-d/n}) is greater than the RHS, which implies that the estimated solution is globally optimal.
Furthermore, the estimated solution is a first-order critical point, and the $2$-norm of the noise matrix is within the proposed bounds. This is consistent with the result of Theorem \ref{final-theo}, demonstrating that the conditions and bounds proposed in this paper can correctly determine the global optimality of the estimated solution.


\begin{figure*}[htbp] 
	\centering 
	\includegraphics[width=0.88\textwidth]{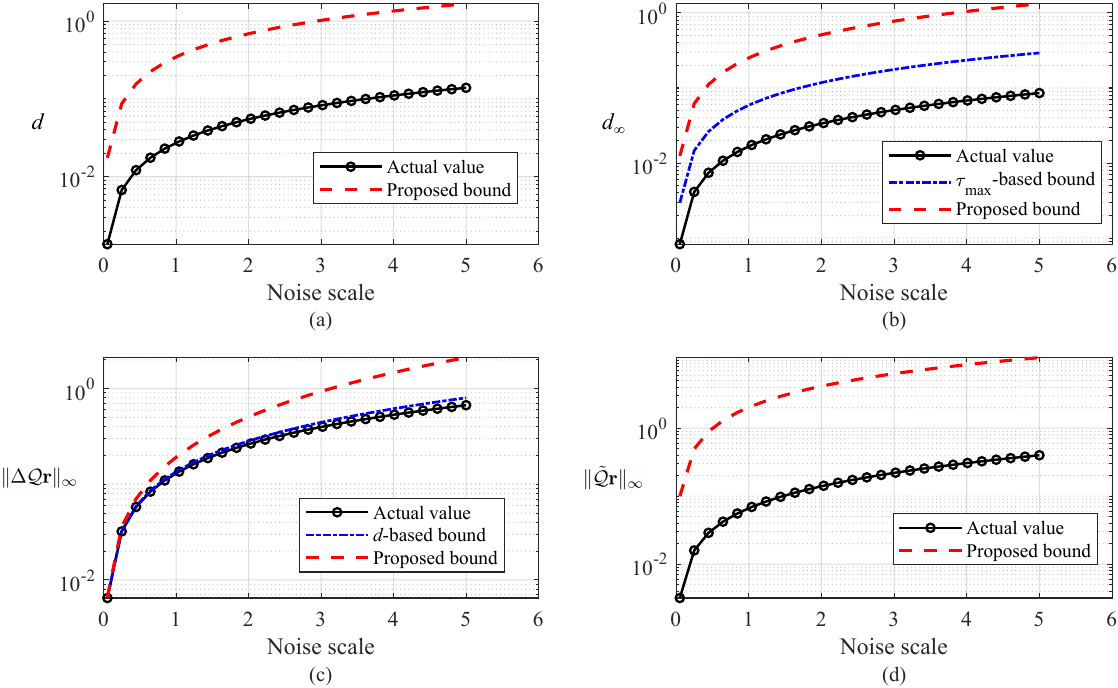} 
	\caption{Evaluation of the proposed bound under varying noise intensities.} 
	\label{fig-bound-vali} 
\end{figure*}

\begin{figure}[htbp] 
	\centering 
	\includegraphics[width=0.44\textwidth]{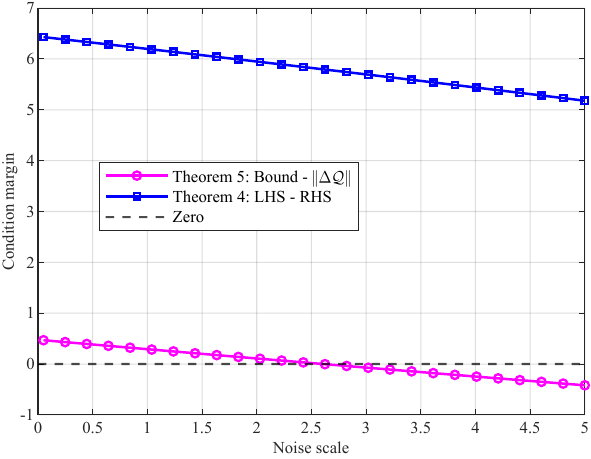} 
	\caption{Margins of  Theorem \ref{lemma-rank-k-n-1} and Theorem \ref{final-theo}.} 
	\label{fig-10-margin} 
\end{figure}

\subsection{Case B: Validation of Theoretical Results}

In this case, we consider a more complex localization problem involving 10 nodes.
We assume Gaussian noise, setting the standard deviation of the baseline position noise to $2 \%$ of the average edge length, and that of the baseline orientation noise to $0.035 \ \si{rad}$. Based on these baseline noise levels, we apply a noise scale factor to adjust the standard deviations, thereby demonstrating the impact of varying noise intensities on the proposed theory.

The simulation results are presented in Fig. \ref{fig-bound-vali} and Fig. \ref{fig-10-margin}.
The Fig. \ref{fig-bound-vali} illustrates the variation trends of the orbit radius (BOA) bound and the bound of noise acting on the estimated solution.
Since Equations (\ref{d_inf-leq-2sq}) and (\ref{deltaQ-r-est-inf-norm}) provide two successively increasing bounds for  $d_{\infty}$ and $	\|\Delta \mathcal Q \mathbf{r}\|_{\infty}$, respectively, these intermediate bounds are all presented in Fig. \ref{fig-bound-vali}(b) and \ref{fig-bound-vali}(c). 
It can be observed that the actual values of the orbit radius and the noise acting are both within the proposed bounds. Furthermore, as the noise scale increases, the  bounds exhibit a steady growth, without any abrupt changes.

For the optimality condition of PGO, Theorem \ref{lemma-rank-k-n-1} and Theorem \ref{final-theo} actually represent two distinct criteria. 
Theorem \ref{lemma-rank-k-n-1} serves as an indirect condition that relies on the actual orbit radius, whereas Theorem \ref{final-theo} provides a more direct criterion by establishing an explicit upper bound on the noise matrix. 
Fig. \ref{fig-10-margin} illustrates the variation trends of these two conditions. 
For presentation clarity, the values are plotted relative to zero. 
Specifically, for Eq. (\ref{gamma-1-d/n}), the value of (LHS - RHS) is depicted; 
for Eq. (\ref{bound-of-delta-Q-norm}), the value of (Bound - $ \|\Delta \mathcal Q\| $)  is shown.
It can be observed that the value of Theorem \ref{lemma-rank-k-n-1} remains consistently greater than zero, stably verifying the optimality of solution. 
In contrast, the value of Theorem \ref{final-theo} gradually drops below zero as the noise scale increases, thereby rendering Theorem \ref{final-theo} invalid. 
This phenomenon is reasonable, as Theorem \ref{lemma-rank-k-n-1} involves no significant relaxation, whereas Theorem \ref{final-theo} introduces upper bounds derived from various relaxations.
However, within a certain range of noise intensity, Theorem \ref{final-theo} can still serve as an efficient criterion for verifying the global optimality of  solution.

\subsection{Case C: Monte-Carlo Validation on Random Instances}

In this case, we consider random instances with varying topologies and noise levels via Monte Carlo sampling to analyze and verify the optimality conditions corresponding to Theorem \ref{lemma-rank-k-n-1} and Theorem \ref{final-theo}. 
A 20-node configuration is considered, where the topological connectivity varies randomly. 
This connectivity is represented by the topology coefficient $\frac{m}{N}$,
with $m$ denoting the number of edges. 
Similar to \textit{Case B} , the same baseline 
noise is adopted, and the noise intensity is adjusted by varying the noise scale factor. 
A total of 8 topology coefficients are evaluated, each corresponding to 7 noise scale factors. 
For each noise scale, 30 random samples are drawn, resulting in a total of 1,680 generated instances. 
As previously mentioned, for each random instance, we conduct the search from multiple initial points and employ dual certificate verification to ensure that the obtained estimated solution is globally optimal.

\begin{figure}[h] 
	\centering 
	\includegraphics[width=0.44\textwidth]{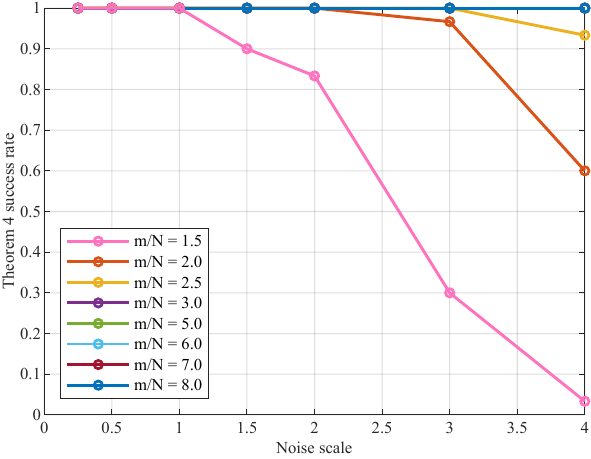} 
	\caption{Success rate of optimality condition in  Theorem \ref{lemma-rank-k-n-1}.} 
	\label{fig-20-T4-succ} 
\end{figure}

\begin{figure}[h] 
	\centering 
	\includegraphics[width=0.44\textwidth]{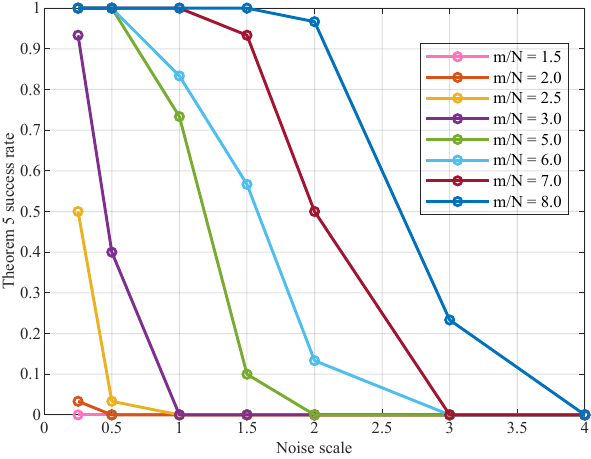} 
	\caption{Success rate of optimality condition in  Theorem  \ref{final-theo}.} 
	\label{fig-20-T5-succ} 
\end{figure}

The simulation results are shown in Figs. \ref{fig-20-T4-succ}, \ref{fig-20-T5-succ}, 
\ref{fig-20-T4-point}, and \ref{fig-20-T5-point}, respectively.
Fig. \ref{fig-20-T4-succ} and Fig. \ref{fig-20-T5-succ} illustrate the success rates of optimality conditions for Theorems \ref{lemma-rank-k-n-1} and \ref{final-theo} under varying topological connectivities and noise intensities. 
It can be observed that a higher number of edge connections leads to a higher success rate, 
whereas an increase in noise intensity results in a decreased success rate. 
This phenomenon is reasonable because topological connectivity fundamentally affects the robustness of the entire system ($\gamma^\star$); 
specifically, a greater number of observation edges reduces the interference of noise on valid observation  information.

\begin{figure}[h] 
	\centering 
	\includegraphics[width=0.44\textwidth]{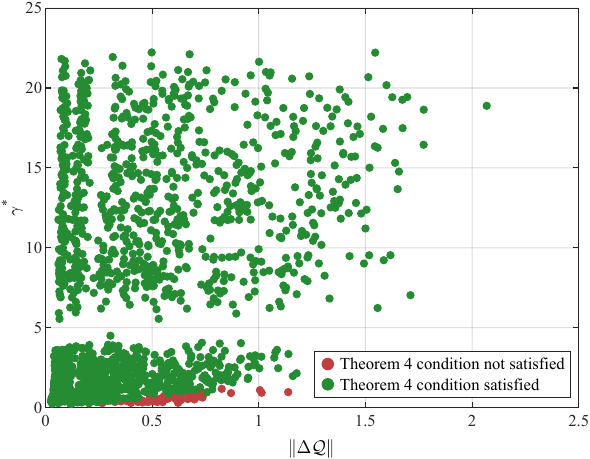} 
	\caption{Distribution of optimality condition satisfaction in Theorem \ref{lemma-rank-k-n-1}.} 
	\label{fig-20-T4-point} 
\end{figure}

\begin{figure}[h] 
	\centering 
	\includegraphics[width=0.44\textwidth]{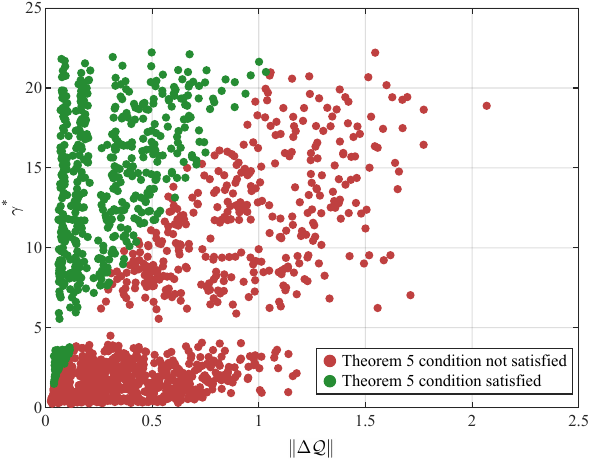} 
	\caption{Distribution of optimality condition satisfaction in  Theorem  \ref{final-theo}.} 
	\label{fig-20-T5-point} 
\end{figure}

Figs. \ref{fig-20-T4-point} and \ref{fig-20-T5-point} illustrate the satisfaction of optimality conditions in Theorems \ref{lemma-rank-k-n-1} and \ref{final-theo} across each random instance. 
It can be observed that Theorem \ref{lemma-rank-k-n-1} fails only in cases with extremely sparse connections, whereas Theorem \ref{final-theo} proves more effective in scenarios with denser connections. 
This observation is consistent with the previously discussed results. 
It is worth noting that since the bounds in Theorem \ref{final-theo} introduce upper bounds for multiple variables, these relaxed bounds inevitably affect the bound on $\|\Delta \mathcal Q\|$, thereby somewhat limiting the applicability. 
Nevertheless, as an explicit upper bound, 
Theorem \ref{final-theo} remains valid and holds significant application potential.

\section{Conclusion}

In this paper, we investigate the BOA  and global optimality of PGO from a theoretical perspective. Based on the complex-domain formulation and its SDP relaxation, we establish a theoretical framework connecting orbit radius, noise perturbation, and global optimality. 
Specifically, we derive explicit upper bounds on the orbit radius within the BOA under various norms.
Furthermore, by constructing a dual certificate based on the SDP relaxation and the KKT conditions, we establish sufficient conditions under which a first-order critical point is guaranteed to be globally optimal.

Building on the derived orbit-radius and noise-action bounds, we further obtain an explicit analytical upper bound on the noise matrix norm 
to guarantee the global optimality.
These results establish a quantitative connection between the geometry of the BOA and the global optimality, revealing a favorable optimization landscape under  small noise. 
Simulation results under various cases verify the proposed theory, demonstrating that it can serve as an effective criterion for analyzing and determining the global optimality of PGO.

\bibliographystyle{IEEEtran}
\normalem\bibliography{ref_paper}

\end{document}